\documentclass[twoside]{article}

\usepackage{aistats2020}
\usepackage{graphicx}
\usepackage{subfigure}
\usepackage{mdframed}
\usepackage{amsmath,amssymb,amsthm}
\usepackage{thmtools,thm-restate}
\usepackage[table]{xcolor}

\DeclareMathOperator*{\argmax}{arg\,max}

\newtheorem{lemma}{Lemma}

\newtheorem*{remark*}{Remark}

\begin{document}

%

%

\twocolumn[

\aistatstitle{On the Convergence of Approximate and Regularized Policy Iteration Schemes}


\aistatsauthor{ Elena Smirnova\\e.smirnova@criteo.com \And Elvis Dohmatob \\e.dohmatob@criteo.com}
\aistatsaddress{ Criteo AI Lab \And Criteo AI Lab } ]

\begin{abstract}
Entropy regularized algorithms such as Soft Q-learning and Soft Actor-Critic, recently showed state-of-the-art performance on a number of challenging reinforcement learning (RL) tasks.
The regularized formulation modifies the standard RL objective and thus generally converges to a policy different from the optimal greedy policy of the original RL problem.
Practically, it is important to control the sub-optimality of the regularized optimal policy.
In this paper, we establish sufficient conditions for convergence of a large class of regularized dynamic programming algorithms, unified under regularized modified policy iteration (MPI) and conservative value iteration (VI) schemes. We provide explicit convergence rates to the optimality depending on the decrease rate of the regularization parameter. 
Our experiments show that the empirical error closely follows the established theoretical convergence rates. In addition to optimality, we demonstrate two desirable behaviours of the regularized algorithms even in the absence of approximations: robustness to stochasticity of environment and safety of trajectories induced by the policy iterates.
\end{abstract}


\section{Introduction}
\label{sec:intro}
The main principle of the entropy regularized approach to RL~\cite{ziebart2008maximum,haarnoja2017reinforcement,fox2015taming,nachum2017bridging} is to modify the standard RL objective to additionally maximize the (relative) entropy of a policy at each visited state. The regularization parameter, usually referred to as \textit{temperature}, controls the relative importance of the entropy term versus the reward. The resulted regularized objective has shown improved robustness to stochastic noise~\cite{fox2015taming} and environment perturbations~\cite{ziebart2008maximum}, as well as better exploration targeted at high-value actions~\cite{haarnoja2017reinforcement}.

Recently, several theoretical frameworks have been proposed to capture a multitude of entropy-regularized methods~\cite{neu2017unified,geist2019theory,kozuno2019theoretical}. 
In this paper, we focus on the regularized MPI scheme~\cite{geist2019theory} that generalizes both value and policy-based regularized methods, such as popular soft actor-critic~\cite{haarnoja2018soft} and soft VI~\cite{haarnoja2017reinforcement}. We also analyse a related value-based scheme, the conservative VI~\cite{kozuno2019theoretical}, that in addition to the soft VI includes the gap-increasing methods, such as advantage learning~\cite{baird1999reinforcement,bellemare2016increasing}.

Despite the empirical success, the regularized algorithms would not generally converge to the optimal policy/value pair. A natural way of controlling the optimality gap is through the regularization weight, which when set at zero recovers the unregularized objective. Thus, a common idea is to gradually decrease the regularization weight over the iterations to eventually converge into this regime.

Prior works considered decaying temperature during learning in the context of specific algorithms. \cite{singh2000convergence} proved asymptotic convergence of SARSA~\cite{sutton1998reinforcement} with Boltzmann policy and decaying temperature. Experimentally, a linear schedule of the inverse temperature over iterations was used with soft Q-learning~\cite{fox2015taming} and dual averaging algorithms~\cite{neu2017unified}. The authors in~\cite[D.1]{geist2019theory} suggested time-varying values of the regularizer weight analogous to the learning rate in the gradient descent approach. 

A number of empirical and theoretical works suggest that the regularization weight schedule changes the behaviour of the algorithms. In the approximate setting with entropy regularizer, \cite[4.2]{kozuno2019theoretical} showed that larger values of temperature parameter induce a higher degree of error-tolerance of the value-based algorithms. 
\cite{haarnoja2018soft} interpreted temperature parameter as controlling the stochasticity of the resulting policy with high values of temperature resulting in policies close to uniform and thus, inducing more exploration.


\subsection{Summary of main contributions} 

In this work, we contribute by analysing convergence of a large class of approximate and regularized dynamic programming algorithms:
\begin{itemize}
    \item[(1)] We derive a convergence rate of the approximate MPI~\cite{scherrer2015approximate} to optimality in terms of the decrease rate of error sequence (Theorem~\ref{thm:error-analysis}),
    \item[(2)] We derive a convergence rate of the regularized MPI~\cite{geist2019theory} to the optimal solution of the non-regularized RL problem through the reduction to the approximate MPI (Theorem~\ref{thm:reg-mpi-convergence}).
    \item[(3)] We derive a convergence rate of the conservative VI~\cite{kozuno2019theoretical} to optimality depending on the temperature decay rate and the gap-increasing factor (Theorem~\ref{thm:cvi-convergence}).
\end{itemize}
One consequence of the result (2) is that if the regularization weight decreases faster than the discount factor (asymptotically), the regularized MPI scheme converges as fast as the exact MPI. Otherwise, if temperature decays at a slower rate, e.g., inverse polynomially in the number of iterations, the algorithm converges proportionally to the decay of the regularization parameter. Thus, our result explicitly relates the speed of convergence to a different behaviour of the algorithm, such as targeted exploration.

Our experiments demonstrate the convergence of the soft VI~\cite{haarnoja2017reinforcement} with temperature decay on a cliff walking domain~\cite{sutton1998reinforcement}. We show that the empirical error follows tightly the established theoretical bounds. In addition, we show two interesting behaviour of the soft VI even in the simplest exact dynamic programming setting:
\begin{itemize}
    \item[(3)] Faster convergence with an increased level of stochasticity of the environment,
    \item[(4)] Safe trajectories induced by the policy iterates.
\end{itemize}
Finally, we show that the empirical convergence of the conservative VI for varying temperature schedule and gap-increasing factor is in line with the established theoretical rates.




\paragraph{Paper organisation.} In Section~\ref{sec:preliminaries} we detail the notations and introduce MPI-based algorithmic schemes. In Section~\ref{sec:contrib} we present our results (1-3), supported by the experiments in Section~\ref{sec:experiments}. We discuss the related works in Section~\ref{sec:related-work}. 

\section{Preliminaries}
\label{sec:preliminaries}
\subsection{Notations and terminology}
$\Delta_X$ will denote the set of probability distributions over finite set (or general measurable space) $X$ and $Y^X$ is a set of  mappings from set $X$ to set $Y$.
We consider a Markov decision process (MDP) is a tuple $M := (\mathcal{S}, \mathcal{A}, P, r, \gamma)$ where $\mathcal{S}$ is a state space, $\mathcal{A}$ is a finite action space, $P \in \Delta_{\mathcal{S}}^{\mathcal{S} \times \mathcal{A}}$ is the transition kernel so that the probability of the environment moving to state $s'$ after the agent takes action $a$ in state $s$ is $P(s'|s,a)$, accompanied by a reward $r(s,a)$ (assumed to be bounded). We define a stochastic stationary policy $\pi \in \Delta^{\mathcal{S}}_{\mathcal{A}}$. We consider the discounted setting with discount factor $\gamma \in [0,1)$. 
We define the Bellman operator $\mathcal{T}^\pi$ for any function $V \in \mathbb{R}^\mathcal{S}$, $\forall s \in \mathcal{S}$ as follows:
\begin{equation}
\label{eq:v-bellman}
\begin{split}
    [\mathcal{T}^\pi V](s) &:= \mathbb{E}_{a \sim \pi(\cdot|s)} \left[ r(s,a) + \gamma \mathbb{E}_{s' \sim P(\cdot|s,a)} [V(s')] \right]\\
    &= r^\pi(s) + \gamma P^\pi(\cdot|s)V,
    \end{split}
\end{equation}
where $r^\pi \in \mathbb R^{\mathcal S}$ and $P^\pi \in \Delta_{\mathcal S}^{\mathcal S}$ are defined by $r^\pi(s) := \mathbb E_{a \sim \pi(\cdot|s)}[r(s,a)]$ and $P^\pi(s'|s) := \mathbb E_{a \sim \pi(\cdot|s)}[P(s'|s,a)]$.
$\mathcal T^\pi$ is a $\gamma$-contraction in $\ell_\infty$ norm and its unique fixed-point is $V^\pi$:= 
$\lim_{k \rightarrow \infty} (\mathcal{T}^\pi)^k V = V^\pi$, where equality holds component-wise.
By denoting 
\begin{equation*}
 Q_V(s,a) := r(s,a) + \gamma \mathbb{E}_{s' \sim p(s'|s,a)} [V(s')],
\end{equation*}
Eq.~\eqref{eq:v-bellman} can be re-written as an inner-product
 $
 [\mathcal{T}^\pi V](s = \langle \pi(\cdot|s),Q_V(s,\cdot)\rangle.
 $

Finally, we define the Bellman max-operator as follows (the max is point-wise)
\begin{equation}
\label{eq:v-opt-bellman}
     \mathcal{T}^\star V := \max_{\pi \in \Delta_{\mathcal{A}}^{\mathcal S}} \mathcal{T}^\pi V, 
\end{equation}
which again is a $\gamma$-contraction in $\ell_\infty$ norm and its unique fixed-point is the optimal value function $V^\star$.
We denote by $\mathcal{G}(V)$ the set of optimal policies that achieve the maximum of Eq.~\eqref{eq:v-opt-bellman} state-wise
\begin{equation*}
    \mathcal{G}(V) := \argmax_{\pi \in \Delta_\mathcal{A}^\mathcal{S}} \mathcal{T}^\pi V \subseteq \Delta_\mathcal{A}^\mathcal{S}.
\end{equation*}
Equivalently, this set coincides with the set of optimal policies: $\mathcal{G}(V) = \{\pi: \mathcal{T}^\pi V = \mathcal{T}^\star V\}$.
\subsection{Modified policy-iteration schemes}
\label{sec:mpi}

\paragraph{Modified Policy Iteration (MPI)~\cite{puterman1994markov}.} MPI is a classical dynamic programming algorithm that alternates between policy improvement and (partial) policy evaluation steps. For $m \geq 1$, the $m$-step MPI algorithm is defined as follows
\begin{equation}
\begin{cases}
\label{eq:mpi}
    \pi_{t+1} \in \mathcal{G}(V_t) \\
    V_{t+1} = (\mathcal{T}^{\pi_{t+1}})^m  V_{t},
\end{cases}
\end{equation}
where $m=1$ corresponds to Value Iteration and $m=\infty$ corresponds to Policy Iteration. Here $V_t$ denotes an approximation of $V^{\pi_t}$.  

\paragraph{Approximate Modified Policy Iteration (AMPI)~\cite{scherrer2015approximate}.} AMPI is an approximate counterpart of~\eqref{eq:mpi} that can be seen as a generalization of MPI that allows errors in the policy improvement ($\epsilon'_t$) and policy evaluation ($\epsilon_t$) steps
\begin{equation}
\label{eq:approx-mpi}
\begin{cases}
    \pi_{t+1} \in \mathcal{G}_{\epsilon'_{t+1}}(V_t)\\
    V_{t+1} =  (\mathcal{T}^{\pi_{t+1}} V_{t})^m + \epsilon_{t+1},
\end{cases}
\end{equation}
where $\epsilon_t, \epsilon'_t \in \mathbb{R}^\mathcal{S}$ are respectively the evaluation step and the policy improvement step error vectors (one component per state) and $\pi \in \mathcal{G}_{\epsilon'}(V) \iff \forall \pi' \ \mathcal{T}^{\pi'}V \leq  \mathcal{T}^\pi V + \epsilon'$. AMPI naturally arises from MPI in practical settings with large state and / or action spaces. 

\paragraph{Convex conjugate functions.} Following~\cite{geist2019theory}, we introduce the regularized RL framework through convex conjugate functions, see e.g. Section 3.3.1 in~\cite{boyd2004convex}. For a strongly convex function $\Omega: \Delta_{\mathcal{A}} \rightarrow \mathbb{R}$ its convex conjugate $\Omega^\star: \mathbb{R}^\mathcal{A} \rightarrow \mathbb{R}$ is given by
\begin{equation}
\label{eq:fenchel-dual}
   \Omega^\star(q) = \max_{\pi_s \in \Delta_{\mathcal{A}}} \langle \pi_s, q \rangle - \Omega(\pi_s),\;\forall q \in \mathbb R^{\mathcal A},
\end{equation}
where $\langle \pi_s,q\rangle := \mathbb E_{a \sim \pi_s}[q(a)]$. $\Omega(\pi)$ will be used as a shorthand for the vector $(\Omega(\pi_s))_{s \in \mathcal S}$.
Further, we make use of the weighted regularizer $\Omega_\alpha(\pi) := \alpha \Omega(\pi)$ that, by properties of the convex conjugate, results in $\Omega^*_\alpha(q) := \alpha \Omega^*(q/\alpha)$. 
Another property of the convex conjugate (\textit{Danskin's Theorem}) is that the maximizer of~\eqref{eq:fenchel-dual} is given by the gradient of the dual function
\begin{equation}
\label{eq:fenchel-grad}
    \nabla \Omega^\star(q) = \argmax_{\pi_s \in \Delta_{\mathcal{A}}}\; \langle \pi_s, q \rangle - \Omega(\pi_s).
\end{equation}

\paragraph{Regularized Modified Policy Iteration (reg-MPI)~\cite{geist2019theory}.}
Similarly to standard Bellman operators~\eqref{eq:v-bellman}, we define the regularized Bellman operator~\cite{geist2019theory}
\begin{equation}
\label{eq:reg-bellman}
    \mathcal{T}^\pi_\Omega V := \mathcal{T}^\pi V - \Omega(\pi),
\end{equation}
and, by~\eqref{eq:fenchel-dual}, the optimal regularized Bellman operator
\begin{equation*}
    \mathcal{T}^*_\Omega V := \max_{\pi \in \Delta_\mathcal{A}^{\mathcal S}} \mathcal{T}^\pi_\Omega V = (\Omega^*(Q_V(s,\cdot)))_{s \in \mathcal S}.
\end{equation*}
By virtue of \eqref{eq:fenchel-grad} the corresponding optimal policy $\mathcal{G}_\Omega(V) \in \Delta_{\mathcal A}^{\mathcal S}$ is given by
\begin{equation*}
\label{eq:reg-optimal-policy}
\mathcal{G}_\Omega(V) := \argmax_{\pi \in \Delta_\mathcal{A}^{\mathcal S}} \mathcal{T}^\pi_\Omega V  = (\nabla \Omega^\star(Q_V(s,\cdot))_{s \in \mathcal S}.
\end{equation*}
Reg-MPI is a MPI-type scheme that underlies several state-of-the-art RL algorithms~\cite{haarnoja2017reinforcement,nachum2017bridging,haarnoja2018soft}
\begin{equation}
\label{eq:reg-mpi}
\begin{cases}
    \pi_{t+1} \leftarrow \mathcal{G}_{\Omega_t}(V_t) \\
    V_{t+1} \leftarrow (\mathcal{T}^{\pi_{t+1}}_{\Omega_t})^m  V_t.
\end{cases}
\end{equation}

\paragraph{Negative entropy regularizer.}
A practically important instance of the reg-MPI scheme corresponds to the negative entropy regularizer
\begin{equation}
\label{eq:entropy-reg}
    \Omega_{\text{Ent}}(\pi(\cdot|s)) = \sum_a \pi(a|s) \log \pi(a|s).
\end{equation}
Further, we will consider this regularization with a time-varying regularization parameter $\lambda_t > 0$
\begin{equation*}
    \Omega_t(\pi(\cdot|s))=\lambda_t\Omega_{\text{Ent}}(\pi(\cdot|s)).
\end{equation*}
Its convex conjugate is the smoothed maximum
\begin{equation}
\label{eq:logsumexp}
    \Omega_t^\star(Q_V(s,\cdot)) = \lambda_t\log \sum_a \exp(Q_V(s,a)/\lambda_t) 
\end{equation}
and the maximizing policy is given by the Boltzmann policy
$\pi_{t+1}(\cdot|s) = \nabla \Omega_t^\star(Q_{V_t}(s,\cdot))$
that takes the form of
\begin{equation}
\label{eq:boltzmann-policy}
    \pi_{t+1}(a|s)=\dfrac{\exp(Q_{V_t}(s,a)/\lambda_t)}{\sum_{a'} \exp(Q_{V_t}(s,a')/\lambda_t)},
\end{equation}
where $\lambda_t > 0$ is referred to as temperature parameter.

With entropic regularization and $m=1$, the reg-MPI scheme~\eqref{eq:reg-mpi} describes the Soft Value Iteration~\cite{haarnoja2017reinforcement} (soft VI)
\begin{equation}
\label{eq:soft-vi}
    V_{t+1} \leftarrow \lambda_t\log \sum_a \exp(Q_{V_t}(s,a)/\lambda_t).
\end{equation}
The asynchronous counterpart of the soft VI is a core principle of the soft Q-learning algorithm~\cite{fox2015taming}.

\section{
Error analysis and convergence rates of AMPI algorithms}
\label{sec:contrib}

We now present the main contributions of this work, namely a fine-grained error analysis of the AMPI-type algorithms, including sufficient conditions for convergence with explicit rates.
\subsection{General AMPI algorithms}
\label{sec:error-analysis}

The error propagation analysis links the error sequence that occurred at previous iterations to the distance to optimality of the current value iterate. In the following Lemma, we restate the error propagation bounds of AMPI established in~\cite[Theorem 7]{scherrer2015approximate} for $p=\infty$.
\begin{lemma}[AMPI error propagation~\cite{scherrer2015approximate}]
\label{thm:scherrer}
For any initial value function $V_0$ and $m \geq 1$, consider the AMPI scheme~\eqref{eq:approx-mpi}. Then, one has 
  \begin{eqnarray}
    \mathcal \|V_N-V^*\|_\infty \le
    \frac{2}{1-\gamma}\left( E_N + \gamma^N\|V_0-V^*\|_\infty \right),
    \label{eq:approx}
 \end{eqnarray}
where $E_N:=\sum_{t=1}^{N-1}\gamma^{N-t}(\|\epsilon_t\|_\infty+\|\epsilon'_t\|_\infty)$.
\end{lemma}

Thus, convergence of the AMPI algorithm entirely depends on controlling the cumulative error term $E_N$. This error term has a special structure in that the errors at later iterations have more contribution to the final loss. In the next theorem, we show the general convergence of AMPI if the sequence of sums of evaluation step and improvement step errors $\|\epsilon_N\|_\infty + \|\epsilon'_N\|_\infty$ converge to zero. By analysing the decrease rate of the error sequence, we provide explicit rates of convergence of the AMPI value iterates to the optimal value function.

\begin{restatable}[AMPI convergence]{theorem}{ampi}
\label{thm:error-analysis}
Suppose the error sequences $(\|\epsilon_N\|_\infty)_N$ and
$(\|\epsilon'_N\|_\infty)_N$ satisfy $\|\epsilon_N\|_\infty + \|\epsilon'_N\|_\infty \le C r_N$ for some constant $C>0$ and a sequence $r_N\longrightarrow 0$.
Then, the AMPI scheme~\eqref{eq:approx-mpi} converges to the optimal greedy
policy of the exact MPI~\eqref{eq:mpi}. 

Furthermore, the limits $\underline{\rho}:={\underline{\lim}\;} r_N/r_{N-1}$ and $\overline{\rho}:={\overline{\lim}\;} r_N/r_{N-1}$. We have the following bounds
\begin{itemize}
\item[\textit{(A)}] \textbf{Slow convergence.} If $\;\underline{\rho} > \gamma$, then
$$
\mathcal \|V_{N}-V^*\|_\infty = \mathcal O(r_N).
$$
\item[\textit{(B)}] \textbf{(Almost) linear convergence.} If $\;\overline{\rho} \le \gamma$, then
  $$\mathcal \|V_{N}-V^*\|_\infty = \begin{cases}\mathcal O(\gamma^N),&\mbox{ if }\;\overline{\rho} <\gamma,\\\mathcal O(N\gamma^N),&\mbox{ if }\;\overline{\rho} = \gamma.\end{cases}
  $$
\end{itemize}
\end{restatable}

We note that the conditions of the Theorem are not restrictive. For example, the maximum error can decrease as slow as inverse logarithmically in the number of iterations and still eventually yield an optimal policy at the rate given by Theorem~\ref{thm:error-analysis}(A). A similar property holds for estimates of the form $r_N \propto 1 / N$; $r_N \propto 1 / \sqrt{N}$; $r_N \propto 1 / \log N$; $r_N \propto \log N / N$; etc. where $\underline{\rho} = \overline{\rho} = 1 > \gamma$. On the other hand, if the error sequences decrease at at rate which is
(asymptotically)
less than the discount factor $\gamma$, then the AMPI converges at the same linear rate as the exact MPI! Moreover, the conditions of Theorem~\ref{thm:error-analysis} allow finitely many deviations of ratios as soon as their inferior limit is bounded away from $\gamma$.
Fig.~\ref{fig:bounds} in Appendix~\ref{sec:appendix} illustrates these bounds.

\subsection{Regularized MPI algorithms}
\label{sec:convergence-reg-mpi}
We first show that the reg-MPI~\eqref{eq:reg-mpi} is an instance of the AMPI~\eqref{eq:approx-mpi}. Then, using Theorem~\ref{thm:error-analysis}, we bound the distance between the value iterates of the reg-MPI and the optimal solution of the exact MPI~\eqref{eq:mpi}. Without loss of generality, we will consider the reg-MPI scheme with weighted regularizer.

\begin{restatable}[Reg-MPI convergence]{theorem}{regmpi}
\label{thm:reg-mpi-convergence}
Consider the reg-MPI algorithm~\eqref{eq:reg-mpi} with time-varying regularization functions $\Omega_t$, and let the sequence $(\lambda_t)_t$ which uniformly bounds $\Omega_t$, that is
\begin{equation}
\label{eq:reg}
\sup_{\pi}\|\Omega_t(\pi)\|_\infty:=\sup_{\pi,\;s}|\Omega_t(\pi(\cdot|s))| \le \lambda_t.
\end{equation}
Then it holds that
\begin{equation}
\label{eq:reg-mpi-upper-bound}
     \|V_{N,\Omega}-V^*\|_\infty \le
    \frac{2}{1-\gamma}\left(\Lambda_N + \gamma^N\|V_{0,\Omega}-V^*\|_\infty \right),
\end{equation}
where $\Lambda_N:=(1 + \frac{1-\gamma^m}{1-\gamma})\sum_{t=1}^{N-1}\gamma^{N-t} \lambda_t$. Moreover, if $\lambda_t \longrightarrow 0$, then the algorithm converges to the optimal value function $V^*$.

Furthermore, define the limits $\underline{\rho}:={\underline{\lim}\;}\lambda_N/\lambda_{N-1}$ and $\overline{\rho}:={\overline{\lim}\;}\lambda_N/\lambda_{N-1}$. We have the following bounds
\begin{itemize}
    \item[(A)] \textbf{Slow convergence.} If $\;\underline{\rho}> \gamma$, then the algorithm converges to the optimal value function $V^*$ with the same rate as the step-sizes:
    $$
    \|V_{N,\Omega}-V^*\|_\infty = \mathcal O(\lambda_N).
    $$
    \item[(B)] \textbf{(Almost) linear convergence.} If $\;\overline{\rho} \le \gamma$, then the algorithm converges to the optimal value function $V^*$ at same rate as the exact MPI~\eqref{eq:mpi} (i.e linear rate of convergence). More precisely, $$\|V_{N,\Omega}-V^*\|_\infty = \begin{cases}\mathcal O(\gamma^N),&\mbox{ if }\;\overline{\rho} < \gamma,\\\mathcal O(N\gamma^N),&\mbox { if }\;\overline{\rho}=\gamma.\end{cases}$$
\end{itemize}
\end{restatable}

\begin{remark*}
It should be noted that the reg-MPI~\eqref{eq:reg-mpi} algorithm above is an instance of the AMPI~\eqref{eq:approx-mpi}, with policy evaluation step error $\epsilon_t := V_{t,\Omega} - (\mathcal{T}^{\pi_{t}})^m V_t$, and policy improvement step error given by $\epsilon_t := \max_\pi \mathcal{T}^\pi V_t - \max_\pi \mathcal{T}_{\Omega_t}^\pi V_t$.
\end{remark*}

We note that prior works established the asymptotic convergence of the soft VI and the soft Q-learning using contraction argument~\cite{haarnoja2017reinforcement} and stochastic approximation tools~\cite{fox2015taming}. We use a different type of analysis that provides finite-time error bounds for a large class of regularized algorithms.

\paragraph{Special case: same base regularizer.} The condition \eqref{eq:reg} is satisfied by time-varying regularizers of the form $\Omega_t = \tau_t\Omega$, for some uniformly bounded $\Omega$. These include the negative entropy regularizer as a special case since $\sup_\pi \|\Omega_{\text{Ent}}(\pi)\|_\infty \leq \log |\mathcal{A}|$. Thus, by controlling the temperature parameter of entropy-regularized algorithm, we can control the error of the value function returned by the algorithm w.r.t. the optimal value function.

\paragraph{Relation to exploration.} Prior works suggest that the temperature parameter $\tau_t$ controls the amount of exploration performed by the policy~\cite{haarnoja2017reinforcement} and the error-tolerance of the algorithm~\cite{kozuno2019theoretical}. From~\eqref{eq:reg-mpi-upper-bound} it is apparent that this desired behaviour is achieved in exchange for slower convergence if the decrease rate is greater than the discount factor $\gamma$. In contrast, if the decrease rate is fast enough (smaller than the discount factor asymptotically), then the above-mentioned properties are acquired with no impact on the convergence rate that remains as fast as the exact MPI! To the best of our knowledge, this is the first result to relate non-asymptotic performance of the MPI to exploration.

One limitation of Theorem~\ref{thm:reg-mpi-convergence} is that it does not provide a specific weight schedule to a problem at hand. E.g., the amount of exploration needed depends on the MDP structure. Too fast temperature decay implies no regularization and leads to insufficient exploration. In contrast, too slow temperature decay results in too strong regularization and unnecessary slow convergence. This trade-off has also been shown empirically on a class of entropy regularized algorithms in~\cite{neu2017unified}.

\subsection{Conservative Value Iteration (CVI)}
\label{sec:cvi-convergence}
We study the convergence rate of the conservative VI~\cite{kozuno2019theoretical} (CVI) scheme. CVI generalizes a number of value-based algorithms such as dynamic policy programming~\cite{azar2012dynamic} and advantage learning~\cite{baird1999reinforcement,bellemare2016increasing}. The CVI algorithm is similar to the soft VI~\eqref{eq:soft-vi} in a sense that it includes the entropic regularization. Differently, it produces Q-value iterates with an increased gap, i.e. amplified difference of Q-values between the maximizing action and all other actions.

Similar to the optimal (regularized) Bellman operators defined in Section~\ref{sec:preliminaries} for value function, we define the optimal (regularized) Bellman operator for Q-values as follows
\begin{eqnarray*}
    \left[\mathcal{T}^* Q\right](s,a) := r(s,a) + \gamma P(\cdot|s,a) \max_{a \in \mathcal{A}} Q(s,a)\\
    \left[\mathcal{T}^*_{\Omega}Q\right](s, a) := r(s,a) + \gamma P(\cdot|s,a) \Omega^*(Q(s, \cdot)),
\end{eqnarray*}
where $\Omega^*(Q(s,\cdot))$ is given by the smoothed maximum~\eqref{eq:logsumexp}.

Using this notation, the CVI scheme as defined by~\cite[Eq.(13)]{kozuno2019theoretical} can be presented in terms of entropic regularizer~\eqref{eq:entropy-reg}, where we additionally vary the regularization weight $\Omega_t := \beta_t \Omega_{\text{Ent}}$
\begin{equation}
\label{eq:cvi}
    Q_{t+1}(s,a) \leftarrow [\mathcal{T}^*_{\Omega_t} Q_t](s, a) + \alpha \left(Q_t(s,a) - \Omega_t^*(Q_t(s,\cdot)) \right),
\end{equation}
where $(\beta_t)_t > 0$ is sequence of weights and $\alpha \in [0,1]$. The second term in~\eqref{eq:cvi} penalizes Q-values of all suboptimal actions to produce a larger difference with the optimal action by a factor of $\alpha$.
The resulting policy $\pi_t$ is given by the Boltzmann policy
\begin{equation*}
\pi_{t+1}(\cdot|s) \leftarrow \mathcal{G}_{\Omega_t}(Q_t(s,\cdot)).   
\end{equation*}

Our analysis of the CVI is based on the analysis of the approximate advantage learning (AL) algorithm
\begin{equation}
\label{eq:approx-al}
\begin{split}
    &Q_{t+1}(s,a) \leftarrow \left[\mathcal{T}^* Q_t\right](s,a) \\
    &+ \alpha \left(Q_t(s,a)-\max_{a \in \mathcal{A}} Q_t(s,a) \right) + \epsilon_t(s,a),
\end{split}
\end{equation}
where $\epsilon_t \in \mathbb{R}^{\mathcal{S}\times\mathcal{A}}$ is an error at iteration $t$. If $\alpha=0$, the approximate AL algorithm coincides with the approximate VI (Eq.~\eqref{eq:approx-mpi} for $m=1$). The following Lemma restates the upper bound on the error of approximate AL established in~\cite[Theorem 1,$\beta=\infty$]{kozuno2019theoretical}.
\begin{lemma}[AL error propagation~\cite{kozuno2019theoretical}]
\label{thm:kozuno}
Consider the approximate AL scheme~\eqref{eq:approx-al}, and let $\Delta Q_N:=Q^{\pi_N}-Q^*$ be the Q-value regret after $N$ iterations. Then, one has 
\begin{equation}
\label{eq:kozuno-cvi}
\begin{split}
    &\|\Delta Q_N\|_\infty \le 2\gamma V_{\max}\Gamma_N \\
     & + \frac{2\gamma}{1-\gamma} \sum_{t=1}^{N} \gamma^{N-t}  \left\|\frac{\sum_{k=0}^t \alpha^{t-k}\epsilon_{k}}{A_N} \right\|_\infty,
\end{split}
\end{equation}
where $A_N := \sum_{k=0}^{N} \alpha^k$ and $\Gamma_N:=\frac{1}{A_N}\sum_{t=0}^{N} \gamma^{N-t}\alpha^t$.
\end{lemma}
If $\alpha=0$, the AL upper bound~\eqref{eq:kozuno-cvi} is $\gamma$ times the upper bound of the approximate MPI~\eqref{eq:approx}, as expected since the bound~\eqref{eq:approx} is given in terms of value function (see Lemma~\ref{thm:q-v-bound}). If $\alpha > 0$, we obtain a different cumulative error term that can be seen as a convolution of powers of $\alpha$ with errors. The error-free term in~\eqref{eq:kozuno-cvi} is of order $\mathcal{O}(N\max(\alpha,\gamma)^N)$ if $\alpha < 1$, and $\mathcal{O}(1/N)$ if $\alpha=1$, compared to $\mathcal{O}(\gamma^N)$ in the approximate MPI~\eqref{eq:approx}.

Next, similar to Section~\ref{sec:convergence-reg-mpi}, we proceed by showing the reduction of the CVI~\eqref{eq:cvi} to the approximate AL~\eqref{eq:approx-al}. Then, we establish the convergence of the CVI through the appropriate error control in the approximate AL.
\begin{restatable}[CVI convergence]{theorem}{cvi}
\label{thm:cvi-convergence}
Consider the CVI algorithm~\eqref{eq:cvi} with time-varying regularization functions $\Omega_t$, and let the sequence $(\lambda_t)_t$ which uniformly bounds $\Omega_t$, that is
\begin{equation*}
    \sup_{\pi}\|\Omega_t(\pi)\|_\infty := \sup_{\pi,\;s}|\Omega_t(\pi(\cdot|s))| \le \lambda_t
\end{equation*}
Let $\Delta Q_N:=Q^{\pi_N}-Q^*$ be the Q-value regret after $N$ iterations. Then it holds that
\begin{equation}
\label{eq:cvi-upper-bound}
    \|\Delta Q_N\|_\infty \le 2\gamma V_{\max}\Gamma_N + \frac{2\gamma}{1-\gamma} \Lambda_N
\end{equation}
where
$\Lambda_N := \frac{1}{A_N} \sum_{t=1}^{N-1} \gamma^{N-t}\left(\sum_{k=0}^t \alpha^{t-k} \lambda_t\right)$, 
$\Gamma_N :=\frac{1}{A_N}\sum_{t=0}^{N} \gamma^{N-t}\alpha^t$, and $A_N:=\sum_{k=0}^{N} \alpha^k$.
Moreover, if $\lambda_t \longrightarrow 0$, then the algorithm converges to the optimal solution $Q^*$.

Furthermore, define the quantities $\bar{\lambda}_N:=\frac{1}{N}\sum_{t=0}^N\lambda_t$,
$$
(\underline{\rho},\overline{\rho}) := \begin{cases}
({\underline{\lim}\;}\lambda_N/\lambda_{N-1},{\overline{\lim}\;}\lambda_N/\lambda_{N-1}),&\mbox{ if }\alpha \ne 1,\\
({\underline{\lim}\;}\bar{\lambda}_N/\bar{\lambda}_{N-1},{\overline{\lim}\;}\bar{\lambda}_N/\bar{\lambda}_{N-1}),&\mbox{ if }\alpha = 1.
\end{cases}
$$
Denote $\alpha \lor \gamma := \max(\alpha,\gamma)$.
Then we have the bounds
\begin{itemize}
    \item[(A)] \textbf{Slow convergence.} If $\alpha=1$ or $\underline{\rho} > \alpha \lor \gamma$,
then
    \begin{eqnarray*}
    \begin{split}
     \|\Delta Q_N\|_\infty &= 
    \begin{cases}
    \mathcal O(\lambda_N),&\mbox{if }\alpha \ne 1,\;\underline{\rho} > \alpha \lor \gamma,\\
    \mathcal O(\bar{\lambda}_N \lor \frac{1}{N}), &\mbox{if }\alpha=1,\;\underline{\rho} > \gamma,\\
    \mathcal O(\frac{1}{N}), &\mbox{if }\alpha=1,\;\overline{\rho} \le \gamma.
    \end{cases}
    \end{split}
    \end{eqnarray*}
    
 \item[(B)] \textbf{(Almost) linear convergence.} If $0 \le \alpha < 1$ and $\overline{\rho} \le \gamma$, then
\begin{eqnarray*}
    \|\Delta Q_N\|_\infty = \mathcal O(N(\alpha \lor \gamma)^N).
\end{eqnarray*}
\end{itemize}
\end{restatable}

\begin{remark*}
It should be noted that the CVI~\eqref{eq:cvi} algorithm above is an instance of the approximate AL~\eqref{eq:approx-al} with error given by $\epsilon_t := \max_\pi \mathcal{T}^\pi Q_t - \max_\pi \mathcal{T}_{\Omega_t}^\pi Q_t$.
\end{remark*}

Theorem~\ref{thm:cvi-convergence} shows that by controlling the sum of weighted regularization terms, we control the distance to optimality. 
As expected, if $\alpha=0$, we get the same rates as for the reg-MPI (Theorem~\ref{thm:reg-mpi-convergence}). If $\alpha \in [0, 1)$ and the temperature decays quickly, the CVI convergence at almost linear rate (bounds (B)). Differently from the reg-MPI convergence, the fast geometric decay of temperature can be slowed down by $\alpha=1$, where the convergence becomes inverse linear in $N$ (bounds (A)). The value of $\alpha=1$ can be beneficial with constant temperature since the CVI would still converge inverse linearly in $N$, whereas the reg-MPI bound~\eqref{eq:reg-mpi-upper-bound} turns to a constant.

\section{Experimental results}
\label{sec:experiments}

We provide an empirical evidence on our theoretical results presented in Section~\ref{sec:contrib}. We experiment with the soft VI, that is an instance of the reg-MPI scheme analysed in Section~\ref{sec:convergence-reg-mpi}, and the conservative VI analysed in Section~\ref{sec:cvi-convergence}. In experiments we use a cliff walking domain described below. 

\subsection{Cliff walking domain}
\label{sec:cliff-walking}
We use a cliff walking domain~\cite{sutton1998introduction} based on a 6x4 grid, also utilised for the analysis of the soft Q-learning algorithm~\cite{fox2015taming}. At each step the agent can move one cell in 4 directions (up, down, left and right). The target cell is located in the lower right corner and it is a terminal state with zero reward. At all other cells, agent receives a reward of~-1 except the bottom row that represents a cliff where the agent is given a reward of~-100. Thus, the goal of the agent is to reach the target cell as quickly as possible and avoid the cliff.
For all experiments, we set the discount factor $\gamma=0.9$.

\subsection{Soft Value Iteration}
\label{sec:exp-svi}
Soft VI~\eqref{eq:soft-vi} is an instance of the reg-MPI scheme with negative entropy regularizer and $m=1$.
We experimentally analyse the convergence of the soft VI~\eqref{eq:soft-vi} with varying temperature schedules. First, we demonstrate the convergence rate and compare it with the established bounds (Theorem~\ref{thm:reg-mpi-convergence}). Next, we analyse properties of the policy iterates. Finally, we experiment with different levels of stochasticity of the environment. We emphasize that we carry on our study in the simplest exact setting. As we shall see, the behaviour of the regularized algorithms is already interesting even in the absence of function approximation~\cite{kozuno2019theoretical} or asynchronous updates~\cite{fox2015taming}.

\begin{figure}[htbp]
    \centering
    \includegraphics[width=\linewidth]{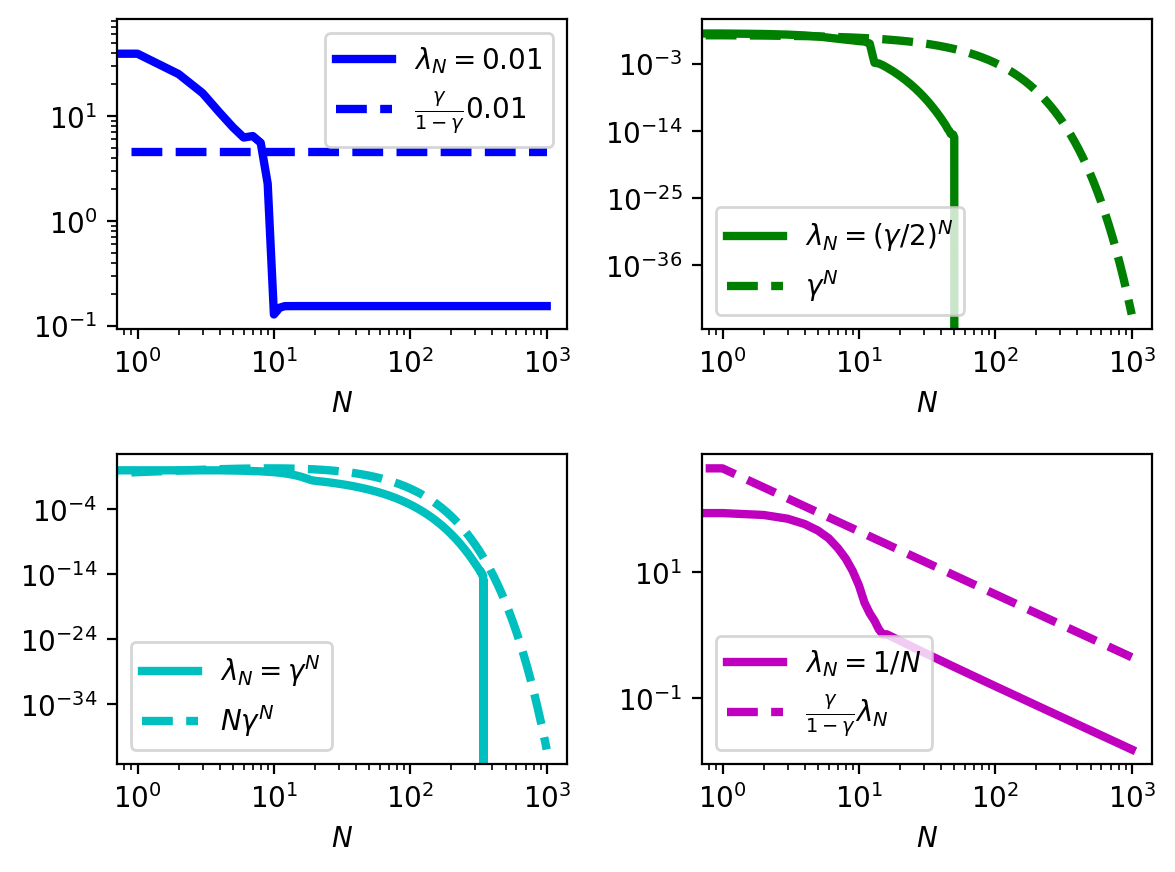}
    \caption{Comparison of upper bounds on the error of the reg-MPI for different convergence regimes (Theorem~\ref{thm:reg-mpi-convergence}) and the empirical error of the soft VI over iterations $N$ on a cliff walking domain. We see that the proposed bounds are really tight.}
    \label{fig:value-bound}
\end{figure}

\begin{figure}[htbp]
    \centering
    \includegraphics[width=\linewidth]{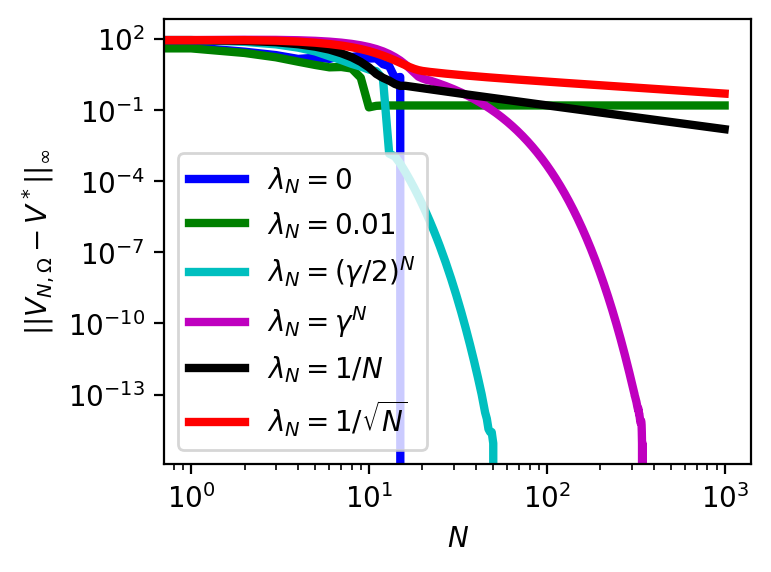}
    \caption{Empirical convergence of the exact VI ($\lambda_N=0$) and the soft VI to optimality with varying temperature schedules on a cliff walking domain follows convergence regimes given by Theorem~\ref{thm:reg-mpi-convergence}.}
    \label{fig:value-temp}
\end{figure}

\paragraph{Convergence.}
We experiment with the temperature schedules corresponding to different regimes of convergence given by the Theorem~\ref{thm:reg-mpi-convergence}. We plug the temperature schedules $\lambda_N$ as a function of the number of iterations $N$ of the soft VI. We also compare to a frequently used fixed $\lambda_N = \lambda = 0.01$ and to the exact VI that corresponds to $\lambda_N = \lambda = 0$. 

First, we plot the maximum empirical error between the current value iterate of the soft VI and the optimal value function with varying temperature schedules, see Figure~\ref{fig:value-temp}. As expected, the exact VI has the fastest convergence. The soft VI with fixed temperature value does not converge to the optimal value function resulting in irreducible error. The convergence of the soft VI is evident for the fastest linear rates of the temperature decay. The inverse linear schedule in $N$ is slower, however it surpasses the fixed value in terms of distance to optimality at around $10^2$ iterations.

Next, we compare the theoretical bounds of the soft VI (Theorem~\ref{thm:reg-mpi-convergence}) with the empirical progression of errors. As can be seen from Figure~\ref{fig:value-bound}, our bounds provide good description of the error in value function at finite time.

\paragraph{Safety of policy iterates.}
We demonstrated above that the soft VI algorithm with temperature decay converges to the optimal value function. 
In the following, we show that (1) intermediate policies of the soft VI with temperature annealing induce safe behaviour and (2) the optimal policy is reached at convergence. Figure~\ref{fig:value-pi-heatmap} compares the progression of the soft VI and the exact VI. As can be seen, the policy iterates of the soft VI avoid the edge of the cliff early in the learning, but the algorithm eventually converges to the optimal trajectory along the cliff. The early iterates of the exact VI directly act optimally following the edge of the cliff.
In a different asynchronous setting with noise, similar observation has been made in~\cite[6.2]{fox2015taming} that in the cliff walking domain the softmax policies with positive temperature result in safe trajectories far away from the cliff. Our analysis shows that this phenomenon is inherent to a fundamentally different behaviour of the soft VI algorithm.

\begin{figure}[htbp]
    \centering
    \includegraphics{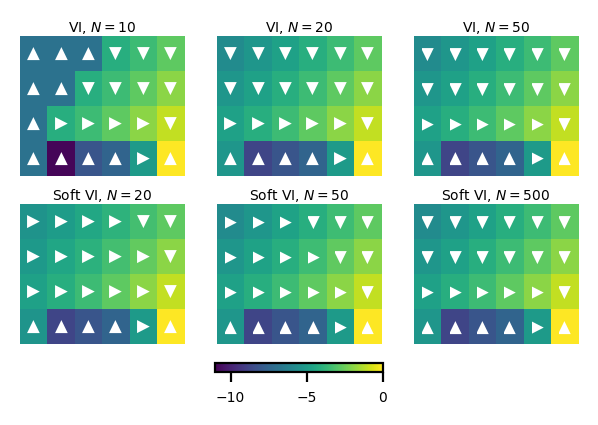}
    \caption{Evolution of value and policy iterates of the exact VI and the soft VI with $\lambda_N = \gamma^N$ on a cliff walking domain. Arrows and color scale indicate, respectively, the best valued action of the policy and the value of a state at iteration $N$.} 
    \label{fig:value-pi-heatmap}
\end{figure}

\paragraph{Robustness to stochasticity.}
We analyse the number of iterations necessary to achieve a certain level of accuracy at different temperature schedules. We study the cliff walking domain from above with added stochastic wind. It consists in replacing the target position by a horizontal or vertical slide of one cell with probability $p$. Table~\ref{tab:nb-iter} shows that on the stochastic cliff walking domain the number of iterations of the soft VI is less than the number of iterations required for the deterministic environment. Moreover, the number of iterations decreases with the amount of added stochasticity. The inverse relation is observed for the exact VI. Thus, we conclude that the soft VI in stochastic environments is beneficial in terms on convergence, resulting, for fast rates of temperature decay, in a comparable number of iterations.

\begin{table}[htbp]
    \centering
    \begin{tabular}{c|c|c|c}
         & \multicolumn{3}{c}{Number of iterations} \\
         & Determ. & \multicolumn{2}{c}{Stochastic} \\
         $\lambda_N$ & $p=0.0$ & $p=0.15$ & $p=0.3$\\
         \hline \hline
         $0$ & 18 & 24 & 31 \\
         \hline
         $(\gamma/2)^N$ & 29 & 27 & 32 \\
         $\gamma^N$ & 166 & 55 & 54 \\
         $1/N$ & 15,691 & 136 & 93 \\
         $1/\sqrt{N}$ & 247,394 & 6379 & 3440 \\
    \end{tabular}
    \caption{Number of iterations of the exact VI ($\lambda_N=0$) and the soft VI necessary to achieve accuracy $\epsilon=10^{-8}$ on deterministic and stochastic cliff walking domains.}
    \label{tab:nb-iter}
\end{table}

\subsection{Conservative Value Iteration}
\label{sec:exp-cvi}
We analyse the empirical convergence of the CVI algorithm with varying temperature decay rates $\rho$ and values of the gap-increasing factor $\alpha$, presented in Table~\ref{tab:cvi-nb-iter}. The color of the cells corresponds to different convergence regimes given by Theorem~\ref{thm:cvi-convergence}. As can be seen from Table~\ref{tab:cvi-nb-iter}, the temperature decay with $\rho \leq \gamma$ has similar convergence rate across values of $\alpha < \gamma$. This matches the almost linear rate $\mathcal{O}(N\gamma^N)$ predicted by Theorem~\ref{thm:cvi-convergence} (B). We observe that fast temperature decay $\rho \leq \gamma$ with large value of $\alpha=0.95$ slows down the convergence, as expected, to $\mathcal{O}(N\alpha^N)$. Slow inverse polynomial temperature schedule implies $\rho > \max(\alpha,\gamma)$ and hence, the convergence is of the same order of magnitude for all values of $\alpha$, predicted as $\mathcal{O}(1/N^2)$ by Theorem~\ref{thm:cvi-convergence} (A). We conclude that Theorem~\ref{thm:cvi-convergence} provides a good description of empirical convergence regimes of the CVI algorithm. We observe that, in line with our theoretical result, if $\alpha=1.$ the CVI does not converge in a reasonable amount of time. 

\begin{table}[htbp]
    \centering
    \begin{tabular}{c|c||c|c|c}
         & & \multicolumn{3}{c}{Number of iterations} \\
         $\lambda_N$ & $\rho$ & $\alpha=0.0$ & $\alpha=0.6$ & $\alpha=0.95$ \\
         \hline \hline
         $0.45^N$ & $0.45$ & \cellcolor[HTML]{ffe6e6}167 & \cellcolor[HTML]{ffe6e6}156 & \cellcolor[HTML]{e6f7ff}399 \\
         $0.8^N$ & 0.8 & \cellcolor[HTML]{ffe6e6}179 & \cellcolor[HTML]{ffe6e6}168 & \cellcolor[HTML]{e6f7ff}399 \\
         $\gamma^N$ & $\gamma$ & \cellcolor[HTML]{ffe6e6}208 & \cellcolor[HTML]{ffe6e6}197 & \cellcolor[HTML]{e6f7ff}399 \\
         $1/N^2$ & $1$ & \cellcolor[HTML]{ebfaeb}1367 & \cellcolor[HTML]{ebfaeb}1058 & \cellcolor[HTML]{ebfaeb}907 \\
    \end{tabular}
    \caption{Number of iterations of the conservative VI necessary to achieve accuracy of $\epsilon=10^{-8}$ with temperature schedule $\lambda_N$, its rate $\rho$ and the gap-increasing factor $\alpha$. Colors signify the convergence regimes predicted by the Theorem~\ref{thm:cvi-convergence}: (A) almost linear $\mathcal{O}(N\gamma^N)$ (red) and $\mathcal{O}(N\alpha^N)$ (blue), (B) slow $\mathcal{O}(1/N^2)$ (green).}
    \label{tab:cvi-nb-iter}
\end{table}




\section{Related works}
\label{sec:related-work}

We first discuss closely related work~\cite{singh2000convergence,pan2019reinforcement}. \cite{singh2000convergence} proves convergence to optimality of the SARSA algorithm with GLIE policies ("greedy in the limit with infinite exploration") that include a class of Boltzmann policies with decaying temperature. Another close work~\cite{pan2019reinforcement} studies convergence to optimality of a value iteration algorithm with a dynamic Boltzmann operator that represents an instance of the reg-MPI scheme with $m=1$, negative entropy regularizer and decreasing temperature schedule. Our work is different from the above-cited work since (1) we consider MPI-based algorithms, (2) our result on convergence rate holds over a class of approximate and regularized MPI algorithms, and (3) we link the desirable properties of the regularized MPI such as targeted exploration to its convergence rate through the schedule of regularization parameter.



A theoretical justification of empirical success of the regularized value-based RL algorithms has been proposed in~\cite{kozuno2019theoretical}. The authors showed that (1) value-based algorithms using softmax operator and fixed temperature parameter are tolerant to any type of errors, e.g. arising from function approximation or due to finite sample of observations used to perform the updates; (2) the temperature parameter controls the trade-off between the asymptotic performance and the sensitivity to errors. Compared to this analysis, our work is complementary in that we provide sufficient conditions of convergence of a family of reg-MPI and conservative VI algorithms. 

The optimization perspective on the regularized MDP framework proposed by~\cite{neu2017unified,geist2019theory} allows the learning rate interpretation of the regularization weight.
In~\cite[D.1]{geist2019theory} the regret of the weighted regularized MPI scheme is analysed when it is subject to approximations. Our work is different in that we consider the regularization itself as errors in the approximate MPI scheme.

The temperature schedules obtained in Section~\ref{sec:error-analysis} have similarities with the decrease factors of the Boltzmann exploration in the multi-arm bandit setting, e.g. $\mathcal{O}(1/N)$ and $\mathcal{O}(\log N/N)$ are frequently used~\cite[2.2]{vermorel2005multi}. Recently, it was shown that temperature schedules of the form $\mathcal{O}(1/\sqrt{N})$ induce near-optimal performance~\cite{cesa2017boltzmann}. Despite these similarities, exploration in the RL setup is not as well understood as in bandits setting; our work contributes by providing a link to the convergence rate.

\section{Conclusion}
Following the success of entropy-regularized methods in RL, we study the convergence to optimality of a class of dynamic programming algorithms unified under the regularized MPI and the conservative VI schemes. By the means of reduction to the approximate counterparts, we showed the general convergence of these schemes to the optimal solution of the original RL problem with decreasing schedule of the regularization parameter over iterations. Moreover, our analysis showed that the convergence of the regularized MPI is as fast as the exact MPI, if the decrease rate of the regularization weight is sufficiently fast; otherwise the algorithm's convergence slows down to the same rate as the decay rate of the regularization parameter. 

We experimentally demonstrate that the empirical convergence closely follows our theoretical results. We showcase a different behaviour of the regularized algorithms even in the absence of approximations, namely, robustness to stochasticity of the environment and safety of trajectories induced by the policy iterates.

\clearpage
\bibliographystyle{unsrt}

\clearpage

\appendix
\onecolumn
\label{sec:appendix}
\section{Proofs}
Refer to the manuscript for the various notations and terminology.

\label{sec:proofs}

The following Lemma will be used repeated in the rest of the proofs.
\begin{figure*}[!ht]
    \centering
    \includegraphics[width=\linewidth]{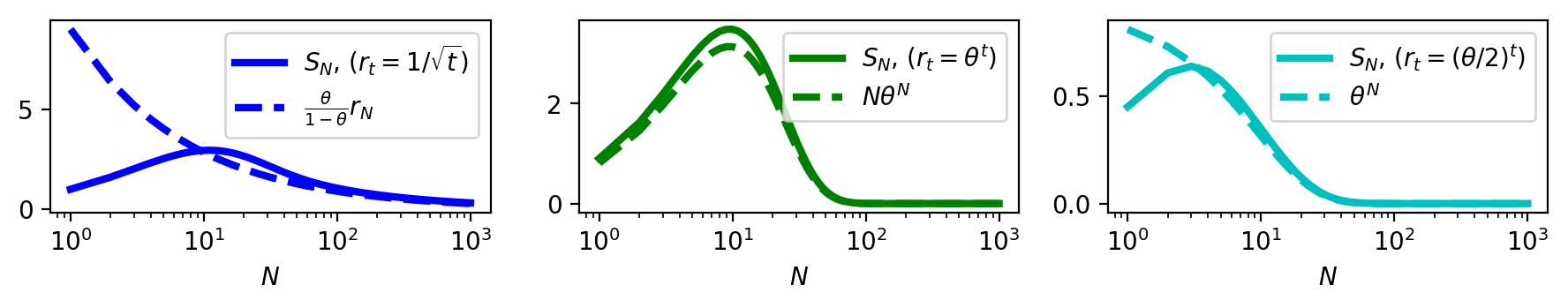}
    \caption{Illustration of the bounds established in Lemma \ref{thm:germ}, for different regimes of the per-iteration error bounds $r_t$. In these illustrations, we plugged $\theta=0.9$. We see that our proposed upper bounds are quite tight.}
    \label{fig:bounds}
\end{figure*}

\begin{lemma}
\label{thm:germ}
Let $r_1,r_2,\ldots,r_t,\ldots$ be a sequence of positive real numbers and $\theta \in [0, 1)$. Define $\underline{\rho}:=\lim\inf r_N/r_{N-1}$ and $\overline{\rho}:=\lim\sup r_N/r_{N-1}$, and consider the sums $S_N := \sum_{t=0}^{N-1}\theta^{N-t}r_t$, for $N \ge 1$. We have the following bounds
\begin{itemize}
    \item[(A)] If $\underline{\rho} \ge \theta$, then
    \begin{eqnarray}
    \label{eq:slow}
    S_N=\mathcal O(r_N).
    \end{eqnarray}
    \item[(B)]  If $\overline{\rho} \le \theta$, then
    \begin{eqnarray}
    \label{eq:fast}
    S_N =
    \begin{cases}
    \mathcal O(\theta^N),&\mbox{ if }\overline{\rho} < \theta,\\
    \mathcal O(N\theta^N),&\mbox{ if }\overline{\rho}=\theta.
    \end{cases}
    \end{eqnarray}
\end{itemize}
\end{lemma}

\begin{proof}
\textit{(A)} Suppose $\underline{\rho} := {\underline{\lim}\;} r_N/r_{N-1} > \theta$. For sufficiently large $t \le N$, we have $r_N \ge \underline{\rho} r_{N-1} \ge \ldots \ge \underline{\rho}^{N-t}r_t$ and so $r_t \le r_N\underline{\rho}^{-(N-t)}$. So, for large $N$, one computes
\begin{eqnarray*}
\begin{split}
S_N := \sum_{t=1}^{N-1}\theta^{N-t}r_t
&\lesssim \sum_{t=1}^{N-1}\theta^{N-t}r_N\underline{\rho}^{-(N-t)}
=r_{N}\sum_{t=1}^{N-1}(\theta/\underline{\rho})^{N-t}
=r_N(\theta/\underline{\rho})\frac{1-(\theta/\underline{\rho})^N}{1-\theta/\underline{\rho}} \lesssim \frac{\theta}{\underline{\rho}-\theta}r_N=\mathcal O (r_N). 
\end{split}
\end{eqnarray*}
\textit{(B)} Suppose $\overline{\rho} := {\overline{\lim}\;} r_N/r_{N-1} < \theta$. Then for sufficiently large $t \le N$, it holds that $r_t \le \overline{\rho} r_{t-1} \le \ldots \le \overline{\rho}^{t-1} r_1$. Thus for sufficiently large $N$, one has
\begin{eqnarray*}
\begin{split}
S_N := \sum_{t=1}^{N-1}\theta^{N-t}r_t \lesssim \sum_{t=1}^{N-1}\theta^{N-t} r_1\overline{\rho}^{t-1}
=r_{1}\theta^{N-1}\sum_{t=0}^{N-2}\left(\frac{\overline{\rho}}{\theta}\right)^t
&=Cr_1\theta^{N-1}\frac{1-(\overline{\rho}/\theta)^{N-1}}{1-\overline{\rho}/\theta}\\
&\lesssim \frac{\theta}{\theta-\overline{\rho}}r_1\theta^{N-1}=\mathcal O (\theta^N). 
\end{split}
\end{eqnarray*}
Finally, for $\overline{\rho}=\theta$, a similar arguments yield
$S_N \lesssim r_1\theta^{N-1}\sum_{t=0}^{N-2}1 = \mathcal O(N\theta^N)$. 
\end{proof}

\subsection{Convergence rates for AMPI (Appproximate Modified Policy Iteration)}
\ampi*
\begin{proof} The proof is based on basic properties of convergent sequences and series.

\textbf{General convergence.} Since $r_t \longrightarrow 0$, it follows that for any $\delta > 0$, $r_t \lesssim \delta$ (where the symbol "$a_t \lesssim b_t$" means that $a_t \le b_t$ for sufficiently large $t$ ). Thus for sufficiently large $N$, one has
\begin{eqnarray*}
\begin{split}
E_N := \sum_{t=1}^{N-1}\gamma^{N-t} (\|\epsilon_t\|_\infty + \|\epsilon'_t\|_\infty) \le C\sum_{t=1}^{N-1}\gamma^{N-t}r_t
\lesssim C\delta\sum_{t=1}^{N-1}\gamma^{N-t} \le C\frac{\gamma}{1-\gamma}\delta.
\end{split}
\end{eqnarray*}
Thus $E_N \longrightarrow 0$ in the limit $N \rightarrow \infty$, and by virtue of the bound \eqref{eq:approx} of Lemma \ref{thm:scherrer} the algorithm converges to the optimal value function $V^*$ as claimed.

\textbf{Convergence with explicit rates.} We now establish the explicit rates of convergence claimed in the theorem under corresponding additional assumptions.

\textit{(A)} Suppose $\underline{\rho} := {\underline{\lim}\;} r_N/r_{N-1} > \gamma$. For sufficiently large $t \le N$, we have $r_N \ge \underline{\rho} r_{N-1} \ge \ldots \ge \underline{\rho}^{N-t}r_t$ and so $r_t \le r_N\underline{\rho}^{-(N-t)}$. So, for large $N$, one computes
\begin{eqnarray*}
\begin{split}
E_N := \sum_{t=1}^{N-1}\gamma^{N-t} (\|\epsilon_t\|_\infty + \|\epsilon'_t\|_\infty) \le C\sum_{t=1}^{N-1}\gamma^{N-t}r_t
&=\mathcal O (r_N),
\end{split}
\end{eqnarray*}
where the last equality an application of Lemma \ref{thm:germ} (more precisely, an application of the bound \eqref{eq:slow} with $\theta=\gamma < \underline{\rho}$).

\textit{(B)} Suppose $\overline{\rho} := {\overline{\lim}\;} r_N/r_{N-1} < \gamma$. Then for sufficiently large $t \le N$, it holds that $r_t \le \overline{\rho} r_{t-1} \le \ldots \le \overline{\rho}^{t-1} r_1$. Thus for sufficiently large $N$, one has
$E_N \le C\sum_{t=1}^{N-1}\gamma^{N-t}r_t = \mathcal O (\gamma^N)$,
by applying Lemma \ref{thm:germ} (more precisely, by applying the bound \eqref{eq:fast} with $\theta=\gamma > \overline{\rho}$).

Finally, for $\overline{\rho}=\gamma$, a similar arguments yield
$E_N = \mathcal O(N\gamma^N)$,
by applying Lemma \ref{thm:germ} (more precisely, by applying the bound \eqref{eq:fast} with $\theta=\gamma = \overline{\rho}$).
\end{proof}

\subsection{Convergence rates for reg-MPI (regularized Modified Policy Iteration)}
\regmpi*
\begin{proof}
We proceed by bounding the policy evaluation and policy improvement step errors of the reg-MPI~\eqref{eq:reg-mpi} with respect to the exact MPI~\eqref{eq:mpi}. We note that reg-MPI~\eqref{eq:reg-mpi} is an instance of AMPI~\eqref{eq:approx-mpi}, with policy evaluation step error $\epsilon_t := V_{t,\Omega} - (\mathcal{T}^{\pi_{t}})^m V_t$, and policy improvement step errors given by $\epsilon_t := \max_\pi \mathcal{T}^\pi V_t - \max_\pi \mathcal{T}_{\Omega_t}^\pi V_t$.

\paragraph{Step 1: bound evaluation-step error $\|\epsilon_t\|_\infty$.} To begin, it is easy to prove by induction on $m$ (see Appendix~\ref{sec:proofs}) that for every policy $\pi \in \Delta_{\mathcal A}^{\mathcal S}$ and value function $V \in \mathbb R^{\mathcal S}$ and one has the formula
\begin{equation}
\label{eq:mstep}
(\mathcal T^{\pi}_\Omega)^mV=(\mathcal T^{\pi})^mV-\sum_{j=0}^{m-1}\gamma^j(P^{\pi})^j\Omega(\pi),
\end{equation}
where $(P^\pi)^j$ is the $j$th power of the matrix $P^\pi$. Thus one has
\begin{eqnarray}
\label{eq:reg-mpi-eval-error}
\begin{split}
\|\epsilon_t\|_\infty &= \|V_{t,\Omega} - (\mathcal{T}^{\pi_{t}})^m V_t\|_\infty
= \|(\mathcal{T}^{\pi_{t}}_{\Omega_t})^m V_t - (\mathcal{T}^{\pi_{t}})^m V_t\|_\infty = \left\| \sum_{j=0}^{m-1} \gamma^j (P^{\pi_{t}})^j\Omega_t(\pi_{t})\right\|_\infty\\
&\leq \sum_{j=0}^{m-1} \gamma^j\|(P^{\pi_{t}})^j\Omega_t(\pi_{t})\|_\infty
\leq \sum_{j=0}^{m-1}\gamma^j\|\Omega_t(\pi_{t})\|_\infty
= \frac{1-\gamma^m}{1-\gamma} \|\Omega_t(\pi_{t})\|_\infty \le \frac{1-\gamma^m}{1-\gamma}\alpha_t,
\end{split}
\end{eqnarray}
where the last inequality follows from the Cauchy-Schwartz inequality \begin{eqnarray*}
\begin{split}
\|(P^{\pi})^j\Omega_t(\pi_{t})\|_\infty = \max_s |(P^{\pi})^j(\cdot|s)\Omega_t(\pi_{t})| \le \max_s \|(P^{\pi_{t}})^j(\cdot|s)\|_1\|\Omega_t(\pi_{t})\|_\infty = \|\Omega_t(\pi_{t})\|_\infty,
\end{split}
\end{eqnarray*}

since $\|(P^{\pi_{t}})^j(\cdot|s)\|_1 =1$ because $(P^{\pi_{t}})^j(\cdot|s)$ is a probability distribution (over next states).

\paragraph{Step 2: bound policy improvement step error $\|\epsilon'_t\|_\infty$.} 
Using elementary properties of the max operator and definition of the regularized operator $\mathcal{T}_\Omega^\pi$, one has
\begin{equation}
\begin{split}
\|\epsilon'_t\|_\infty = \|\max_\pi \mathcal{T}^\pi V_t - \max_\pi \mathcal{T}_{\Omega_t}^\pi V_t\|_\infty
\le \max_\pi\|\mathcal{T}^\pi V_t - \mathcal{T}_{\Omega_t}^\pi V_t\|_\infty
= \max_\pi\|\Omega_t(\pi)\|_\infty \le \alpha_t.
\label{eq:reg-mpi-max-error}
\end{split}
\end{equation}

By combining per-iteration error bounds~\eqref{eq:reg-mpi-eval-error} and~\eqref{eq:reg-mpi-max-error} and using Lemma~\ref{thm:scherrer}, one obtains~\eqref{eq:reg-mpi-upper-bound}. From this bound and Theorem~\ref{thm:error-analysis} invoked with $r_t := \alpha_t$ and $C=1+\frac{1-\gamma^m}{1-\gamma}$, we get that the algorithm reg-MPI~\eqref{eq:reg-mpi} converges to the optimal value function $V^*$, with the claimed rates of convergence.
\end{proof}

\paragraph{Proof of formula \eqref{eq:mstep}.}
Let $\pi$ be a policy and $V$ be a value function.
By ~\eqref{eq:reg-bellman}, one has
$$
\mathcal T_\Omega^\pi V =  \mathcal T^\pi V - \Omega(\pi)=\mathcal T^\pi V - \gamma^0(P^\pi)^0\Omega(\pi),
$$
and so the formula is valid for $m=1$ step. Now suppose the formula~\eqref{eq:mstep} is valid for $m$ steps. Then
\begin{equation*}
\begin{split}
    &(\mathcal T^\pi_\Omega)^{m+1}V = \mathcal T_\Omega^\pi((\mathcal T_\Omega^\pi)^m V) = \mathcal T^\pi((\mathcal T_\Omega^\pi)^m V)-\Omega(\pi)
    =r^\pi + \gamma P^\pi(\mathcal T_\Omega^\pi)^m V  - \Omega(\pi)\\
    &\quad=r^\pi + \gamma P^\pi\left((\mathcal T^\pi)^mV-\sum_{j=0}^{m-1}\gamma^j(P^\pi)^j\Omega(\pi)\right)-\Omega(\pi)
    =r^\pi + \gamma P^\pi(\mathcal T^\pi)^mV-\gamma P^\pi\sum_{j=0}^{m-1}\gamma^j(P^\pi)^j\Omega(\pi)-\Omega(\pi)\\
    &\quad=\mathcal T^\pi((\mathcal T^\pi)^{m}V)-\sum_{j=0}^{m}\gamma^j(P^\pi)^j\Omega(\pi)
    = (\mathcal T^\pi)^{m+1}V-\sum_{j=0}^{m}\gamma^j(P^\pi)^j\Omega(\pi),
    \end{split} 
\end{equation*}
which is the formula~\eqref{eq:mstep} for $m+1$ steps.
\qed
\subsection{Convergence rates for CVI (Conservative Value Iteration)}
\cvi* 
\begin{proof}
The first term $\Gamma_N$ in the upper-bound in Lemma~\ref{thm:kozuno} is itself upper-bounded as follows
\begin{equation*}
    \Gamma_N = \begin{cases} \mathcal{O}(N(\alpha \lor \gamma)^N),&\mbox{ if }0\leq\alpha < 1\\\mathcal{O}(\frac{1}{N}),&\mbox{ if }\alpha=1.  \end{cases}
\end{equation*}
So it remains to control the second term $\Lambda_N$ in the bound.

We first bound the errors $\epsilon_t := \max_\pi \mathcal{T}^\pi V_t - \max_\pi \mathcal{T}_{\Omega_t}^\pi V_t$ in the approximate AL~\eqref{eq:approx-al}. Similar to the Step 2 in the proof of Theorem~\ref{thm:reg-mpi-convergence}, one has
\begin{eqnarray*}
\|\epsilon_t\|_\infty = \|\max_\pi \mathcal{T}^\pi V_t - \max_\pi \mathcal{T}_{\Omega_t}^\pi V_t\|_\infty
\leq \max_\pi\|\Omega_t(\pi)\|_\infty \leq \lambda_t.
\end{eqnarray*}
So, by 
the triangular inequality, we have the bound
\begin{equation*}
    \sum_{t=1}^{N-1} \gamma^{N-t} \left\|\sum_{k=0}^t \alpha^{t-k}\epsilon_k \right\|_\infty \leq \sum_{t=1}^{N-1} \gamma^{N-t} \left(\sum_{k=0}^t \alpha^{t-k} \|\epsilon_k\|_\infty\right) = \sum_{t=1}^{N-1} \gamma^{N-t}\left(\sum_{k=0}^t \alpha^{t-k} \lambda_t\right) =:\Lambda_N.
\end{equation*}
The inner and outer sums are of the same type that we analysed in Lemma~\ref{thm:germ}.

\paragraph{Case 1: $0 \le \alpha < 1$.}
In this case $A_{N} = \Omega(1)$. We will consider different subcases. Viz,
\paragraph{Case 1.1: $\underline{\rho} > \alpha$.} Under this assumption, we have $\sum_{k=0}^t \alpha^{t-k} \lambda_k = \mathcal O(\lambda_t)$ by applying Lemma \ref{thm:germ} (more precisely, by applying the bound \eqref{eq:slow} with $\theta=\alpha < \underline{\rho}$). Thus,
$$
\Lambda_N := \frac{1}{A_N} \sum_{t=1}^{N-1} \gamma^{N-t}\left(\sum_{k=0}^t \alpha^{t-k} \lambda_t\right) \lesssim \frac{1}{A_N}\sum_{t=0}^{N-1}\gamma^{N-t}\lambda_t = \begin{cases}\mathcal O(\lambda_N),&\mbox{ if }\underline{\rho}>\gamma,\\
\mathcal (N\gamma^N),&\mbox{ if }\overline{\rho} \le \gamma,\end{cases}
$$
where the last equality is via another application of Lemma \ref{thm:germ}.
Thus,
$$
\|\Delta Q_N\|_\infty = \begin{cases}
\mathcal O(\lambda_N \lor (N(\alpha \lor \gamma)^N)),&\mbox{ if }\underline{\rho} > \gamma,\\
\mathcal O(N(\alpha \lor \gamma)^N),&\mbox{ if }\overline{\rho} \le \gamma.
\end{cases}
$$

\paragraph{Case 1.2: $\overline{\rho} < \alpha$.} Under this assumption, we have
$\sum_{k=0}^t \alpha^{t-k} \lambda_k = \mathcal O(\alpha^t)$
by applying Lemma \ref{thm:germ} (more precisely, by applying the bound \eqref{eq:fast} with $\theta=\alpha > \overline{\rho}$). Thus,
\begin{eqnarray*}
\begin{split}
\Lambda_N := \frac{1}{A_N} \sum_{t=1}^{N-1} \gamma^{N-t}\left(\sum_{k=0}^t \alpha^{t-k} \lambda_t\right) \lesssim \sum_{t=0}^{N-1}\gamma^{N-t}\alpha^t& = \begin{cases}\mathcal O(\alpha^N),&\mbox{ if }\alpha >\gamma,\\
\mathcal O(\gamma^N),&\mbox{ if }\alpha < \gamma,\\
\mathcal (N\gamma^N),&\mbox{ if }\alpha=\gamma.\end{cases}\\
&= \mathcal O(N(\alpha \lor \gamma)^N),
\end{split}
\end{eqnarray*}
where the last equality is via another application of Lemma \ref{thm:germ}.
Thus
$\|\Delta Q_N\|_\infty = \mathcal O(N(\alpha \lor \gamma)^N))$.

\paragraph{Case 1.3: $\overline{\rho} = \alpha$.} Under this assumption, we have
$\sum_{k=0}^t \alpha^{t-k} \lambda_k =
\mathcal O(t\alpha^t)$
by applying Lemma \ref{thm:germ} (more precisely, by applying the bound \eqref{eq:fast} with $\theta=\alpha = \overline{\rho}$). Thus,
\begin{eqnarray*}
\begin{split}
\Lambda_N := \frac{1}{A_N} \sum_{t=1}^{N-1} \gamma^{N-t}\left(\sum_{k=0}^t \alpha^{t-k} \lambda_t\right) \lesssim \sum_{t=0}^{N-1}\gamma^{N-t}t\alpha^t &= \begin{cases}\mathcal O(N\alpha^N),&\mbox{ if }\alpha >\gamma,\\ \mathcal O(\gamma^N),&\mbox{ if }\alpha < \gamma,\\
\mathcal (N\gamma^N),&\mbox{ if }\alpha=\gamma\end{cases}\\
&=\mathcal O(N(\alpha \lor \gamma)^N)
\end{split}
\end{eqnarray*}
where the last equality is via another application of Lemma \ref{thm:germ}. Thus
$\|\Delta Q_N\|_\infty = \mathcal O(N(\alpha \lor \gamma)^N))$.

\paragraph{Case 2: $\alpha=1$.} Under this assumption, we have
$\sum_{k=0}^t \alpha^{t-k} \lambda_k = \sum_{k=0}^t \lambda_k =: t\bar{\lambda}_t$ and $A_N = N$. Thus
$$
\Lambda_N := \frac{1}{A_N} \sum_{t=1}^{N-1} \gamma^{N-t}\left(\sum_{k=0}^t \alpha^{t-k} \lambda_t\right) = \frac{1}{N}\sum_{t=0}^{N-1}\gamma^{N-t}t\bar{\lambda}_t =
\begin{cases}
\mathcal O\left(\bar{\lambda}_N\right), &\mbox{ if }\underline{\rho} > \gamma,\\
\mathcal O(\gamma^N), &\mbox{ if }\overline{\rho} \le \gamma.
\end{cases}
$$
by applying Lemma \ref{thm:germ} (more precisely, by applying the bound \eqref{eq:fast} with $r_t=t\bar{\lambda}_t$ and $\theta=\gamma$).
Thus,
$$
\|\Delta Q_N\|_\infty = \begin{cases}
\mathcal O((\bar{\lambda}_N \lor 1/N)),&\mbox{ if }\underline{\rho} > \gamma,\\
\mathcal O(1/N),&\mbox{ if }\overline{\rho} \le \gamma.
\end{cases}
$$

Applying Lemma \ref{thm:kozuno} and grouping the various convergence rates of $\Lambda_N$ then yields the bounds on $\|Q^{\pi_N}-Q^*\|_\infty$ claimed in the Theorem.
\end{proof}

\subsection{Relation between error in Q function and error in value function}
\begin{lemma}
\label{thm:q-v-bound}
For Q functions $Q_1, Q_2 \in \mathbb{R}^{\mathcal{S} \times \mathcal{A}}$ with associated value functions $V_1,V_2 \in \mathbb{R}^{\mathcal{S}}$, it holds that
\begin{equation*}
    \|Q_1 - Q_2\|_\infty \leq \gamma \|V_1 - V_2\|_\infty.
\end{equation*}
\end{lemma}
\begin{proof}
For all $s,a \in \mathcal{S} \times \mathcal{A}$
\begin{eqnarray*}
\begin{split}
    |Q_1(s,a) - Q_2(s,a)|& = |r(s,a) - \gamma P(\cdot|s,a)^T V_1 - r(s,a) - \gamma P(\cdot|s,a)^T V_2| \\
    &= \gamma |P(\cdot|s,a)^T(V_1 - V_2)| \leq \gamma \|P(\cdot|s,a)\|_1 \|V_1 - V_2\|_\infty = \gamma \|V_1 - V_2\|_\infty,
\end{split}
\end{eqnarray*}
where the inequality is due to application of Cauchy-Schwartz inequality and we also used the fact that $\|P(\cdot|s,a)\|_1 = 1$ for all for all $s,a \in \mathcal{S} \times \mathcal{A}$ since $P$ is a transition matrix.
\end{proof}

\end{document}